\documentclass[letterpaper]{article}
\usepackage{aaai2026} 
\usepackage{times} 
\usepackage{helvet} 
\usepackage{courier} 
\usepackage[hyphens]{url} 
\usepackage{graphicx} 
\usepackage{graphicx} 
\usepackage{natbib} 
\usepackage{caption} 
\usepackage{algorithm}
\usepackage{algorithmic}
\usepackage{amsmath}
\usepackage{amssymb}
 \usepackage{amsthm}
\usepackage{booktabs}
 \theoremstyle{definition}
\newtheorem{definition}{Definition}
\theoremstyle{theorem}
\newtheorem{theorem}{Theorem}
\newtheorem{lemma}{Lemma}

\newtheorem{property}{Property}

\newcommand{\ADV}{\texttt{ADV}}
\newcommand{\MIA}{\texttt{MIA}}
\newcommand{\GAN}{\texttt{GAN}}
\newcommand{\Diff}{\texttt{Diff}}
\newcommand{\defn}{\ensuremath{:  =}}
\usepackage{multirow}
\usepackage{xcolor}

\usepackage{newfloat}
\usepackage{listings}
\DeclareCaptionStyle{ruled}{labelfont=normalfont,labelsep=colon,strut=off} 
\floatstyle{ruled}
\newfloat{listing}{tb}{lst}{}
\floatname{listing}{Listing}
\title{On the MIA Vulnerability Gap Between Private GANs and Diffusion Models}
\author{
    Ilana Sebag\textsuperscript{\rm 1,2}, 
    Jean-Yves Franceschi\textsuperscript{\rm 1}, 
    Alain Rakotomamonjy\textsuperscript{\rm 1}, 
    Alexandre Allauzen\textsuperscript{\rm 2,3}, 
    Jamal Atif\textsuperscript{\rm 2}
}
\affiliations{
    \textsuperscript{\rm 1} Criteo AI Lab, Paris, France \\
    \textsuperscript{\rm 2} Miles Team, LAMSADE, Université Paris-Dauphine, PSL University, CNRS, Paris, France \\
    \textsuperscript{\rm 3} ESPCI PSL, Paris, France \\
}

\usepackage{bibentry}

\begin{document}

\maketitle

\begin{abstract}
Generative Adversarial Networks (GANs) and diffusion models have emerged as leading approaches for high-quality image synthesis. While both can be trained under differential privacy (DP) to protect sensitive data, their sensitivity to membership inference attacks (MIAs), a key threat to data confidentiality, remains poorly understood. In this work, we present the first unified theoretical and empirical analysis of the privacy risks faced by differentially private generative models. We begin by showing, through a stability-based analysis, that GANs exhibit fundamentally lower sensitivity to data perturbations than diffusion models, suggesting a structural advantage in resisting MIAs. We then validate this insight with a comprehensive empirical study using a standardized MIA pipeline to evaluate privacy leakage across datasets and privacy budgets. Our results consistently reveal a marked privacy robustness gap in favor of GANs, even in strong DP regimes, highlighting that model type alone can critically shape privacy leakage.

\end{abstract}


\section{Introduction}

Generative models have become crucial in machine learning. Among leading generative architectures, GANs \cite{goodfellow2014generativeadversarialnetworks} and diffusion models \cite{ho2020denoisingdiffusionprobabilisticmodels, song2021scorebasedgenerativemodelingstochastic, karras2022elucidatingdesignspacediffusionbased} dominate high-fidelity image synthesis. As these models are increasingly deployed in sensitive domains, their ability to memorize and reproduce training data raises serious privacy concerns, making protection against data leakage essential.

Differential Privacy (DP) \citep{inproceedingsdwork6cr, articledworkprivacyacm11, 10.1561/0400000042} provides a rigorous framework to mitigate this risk, ensuring that a model's output is statistically indistinguishable when altering a single training data point. In practice, DP is commonly implemented via differentially private stochastic gradient descent (DP-SGD) \citep{Abadi_2016}, which clips per-sample gradients and adds calibrated noise during training.

While both GANs and diffusion models can be trained with DP-SGD, their vulnerability to membership inference attacks (MIAs), which aim to determine whether a given sample was used during training, remains poorly understood \citep{shokri2017membership, carlini9833649}.
Moreover, empirical findings in the non-private setting suggest that GANs leak less membership information than diffusion models \citep{carlini2023extractingtrainingdatadiffusion}, raising a central yet unresolved question: \emph{does this gap persist under formal privacy training, and if so, why}?

In this work, we present the first unified theoretical and empirical study of membership leakage in differentially private generative models. To our knowledge, no prior work has analyzed how the training procedure affects data leakage under DP in the context of MIAs. We show in particular that DP-diffusion models are more vulnerable to such attacks than DP-GANs.

Our analysis builds on the notion of uniform stability \cite{bousquet2002stability, hardt2016trainfastergeneralizebetter}, which quantifies how much a model’s output changes when a single training point is replaced. We formally relate this stability to membership inference risk by bounding the adversarial advantage in terms of the model’s stability constant.
Crucially, we show that model’s stability is determined by its training dynamics. While DP-GANs apply DP-SGD only to the discriminator, diffusion models use DP-SGD to train a denoiser under a weighted multi-pass denoising objective. The main source of instability in diffusion models lies in the large loss weights assigned to low-noise denoising terms, which amplify the effect of small parameter changes. As a result, we prove that DP-diffusion models exhibit significantly lower stability and therefore leak more membership information under the same privacy budget.

We validate these insights empirically using a standardized evaluation pipeline. We train multiple instances of GANs and diffusion models in under comparable conditions, notably with the same DP-SGD mechanism and privacy budget $\varepsilon$, and conduct attacks using a consistent shadow-model framework \cite{shokri2017membership}, relying on loss or logits based scoring in a black-box setting. Beyond validating the theory, this constitutes (to our knowledge) the first systematic assessment of membership leakage in differentially private generative models; prior work introducing DP-GANs and DP-diffusion has not assessed their vulnerability to membership inference. 
  
Under identical privacy budgets, we observe consistent gaps in leakage between DP-GANs and DP-diffusion, indicating that the privacy parameter $\varepsilon$ alone does not fully characterize risk. Training architecture is a critical, yet often overlooked, factor of privacy leakage in differentially private generative models. Despite typically higher sample quality, diffusion models exhibit greater membership leakage under DP, whereas GANs are more robust, highlighting a trade-off between fidelity and privacy that has been largely overlooked. These results highlight the importance of evaluating private generative models not only in terms of output quality or reported $(\varepsilon,\delta)$ values, but also through architecture-driven stability and empirical leakage metrics, which provide complementary insights into privacy risk.

\section{Background and Related Work}

\paragraph{Notations.}
Let $\mathcal{X}$ denote the input space and \(\mathcal{Y}\) the output space.  Let \(D = \{x_i\}_{i=1}^m \in \mathcal{D}\) be the training dataset drawn i.i.d.\ from an unknown distribution \(\mathcal{P}\).
A learning algorithm is a map $f : \mathcal{D} \to \mathcal{F} \subset \mathcal{Y}^\mathcal{X}$, that assigns to each training set a hypothesis \(f_D \in \mathcal{F}\), where \(f_D : \mathcal{X} \to \mathcal{Y}\) is the learned model. We assume \(f\) is symmetric with respect to the ordering of samples. For any \(i \in \{1, \dots, m\}\), we write \(D^{\setminus i} = D \setminus \{x_i\} \in \mathcal{D}\) for the neighboring dataset obtained by removing one example. We denote by \(\ell(f, z)\) the per-sample training loss incurred by model \(f\) on example \(z\), and by \(s_f(x)\) a scalar \emph{attack score} computed from the model \(f\) on input \(x\).

\subsection{Differential Privacy}
\label{subsec: DP prelims}

\begin{definition}
    A random mechanism $\mathcal{M}\colon \cal D \rightarrow \cal R$  is $(\varepsilon, \delta)$-DP if  for any two adjacent datasets $D, D' \in \cal D$ differing in at most one element and any for any subset of outputs $\mathcal{S} \subseteq \cal R$, 
    \begin{align}
    \mathbb{P}[ \mathcal{M}(D) \in \mathcal{S}] \leq \mathrm{e}^\varepsilon \mathbb{P}[\mathcal{M}(D') \in \mathcal{S}] + \delta.
\end{align}
\end{definition} 

Adjacent inputs refer to datasets differing only by a single record. DP ensures that when a single record in a dataset is swapped, the change in the distribution of model outputs will be controlled by $\varepsilon$ and $\delta$. $\varepsilon$ controls the trade-off between the level of privacy and the usefulness of the output, where smaller $\varepsilon$ values offer stronger privacy but potentially lower utility (e.g.\ in our specific case, low-quality generated samples). 

A classical example of a DP mechanism is the Gaussian mechanism operating on a function $f\colon \mathcal{D} \to {\mathbb{R}}^d$ as:
\begin{align}
    \mathcal{M}_f(D) = \mathcal{N}(f(D), \sigma^2\mathcal{I}_d).
\end{align}
We define the $\ell_2$ sensitivity of $f$ as  $\Delta_2(f) \defn \max_{D,D' \text{: adjacent}\in \mathcal{D} } \| f(D)-f(D')\|_2$. 
For $c^2 > 2\ln(1.25/\delta)$ and $\sigma \geq c \frac{\Delta_2(f)} {\varepsilon}$, the Gaussian mechanism is $(\varepsilon, \delta)$-DP \citep{10.1561/0400000042}.
In deep learning, differential privacy is most commonly enforced using \emph{Differentially Private Stochastic Gradient Descent} (DP-SGD), introduced by \citet{Abadi_2016}. The goal of DP-SGD is to ensure that each training example has a limited influence on the learned model. To achieve this, each training step involves computing per-sample gradients, clipping their norms to a fixed threshold, adding Gaussian noise, and performing a standard SGD update. This mechanism guarantees \((\varepsilon, \delta)\)-DP over the course of training, where the overall privacy loss is tracked using a composition accountant such as the moment accountant (introduced in the same work). The privacy budget \(\varepsilon\) accumulates over iterations and depends on the batch size, number of steps, and  the noise scale $\sigma$.

\subsection{Differentially Private GANs and Diffusion Models}
\label{subsec:dp_gans_diff}

Differential privacy has recently been applied to generative models, with most approaches relying on DP-SGD to perturb gradient updates during training. This subsection reviews how DP-SGD is integrated into two leading generative frameworks: GANs and diffusion models. In GANs \cite{goodfellow2014generativeadversarialnetworks}, adversarial objectives are optimized under privacy constraints \citep{xie2018differentially, NEURIPS2020_9547ad6b, NEURIPS2021_171ae1bb}, while diffusion models are adapted by injecting noise into gradient updates across denoising steps \citep{dockhorn2023differentially, ghalebikesabi2023differentiallyprivatediffusionmodels}. However, how DP-SGD interacts with these training procedures and impacts privacy leakage remains poorly understood. We address this question through the lens of membership inference and algorithmic stability.

\subsubsection{GANs under Differential Privacy.}
A GAN consists of a generator \( G_\phi(z) \) that maps latent vectors \( z \sim p_z \) to samples, and a discriminator \( D_\psi(x) \) that distinguishes real from generated data. In our setting, \( D_\psi(x) \in \mathbb{R} \) returns a logit, with the sigmoid \( \sigma(u) = 1 / (1 + e^{-u}) \) applied inside the loss function. Their parameters $\phi$ and $\psi$ are optimized through the following minmax objective:
\begin{align}
\nonumber \min_{\phi} \max_{\psi} \ &\mathbb{E}_{x \sim p_{\text{data}}}[\log \sigma(D_\psi(x))] \\
&+  \mathbb{E}_{z \sim p_z}[\log(1 - \sigma(D_\psi(G_\phi(z))))]
\label{eq:gan-objective}
\end{align}

 In the DP setting, only the discriminator accesses real data and is trained with DP-SGD. The generator receives updates exclusively through backpropagation from the discriminator, and thus qualifies as post-processing. As a result, its updates incur no additional privacy cost. This decoupling is key to the relative robustness of GANs under DP, and is further explored in \citet{bie2023privategansrevisited}, which proposes techniques for stabilizing GAN training in this regime.

\subsubsection{Diffusion Models under Differential Privacy.}
Diffusion models \cite{song2020generativemodelingestimatinggradients} synthesize data by reversing a stochastic process that gradually corrupts clean images with Gaussian noise. Given a sample \( x_0 \sim p_{\text{data}} \),  the forward process generates noisy inputs \( x_\sigma \) as: $x_\sigma = x_0 + \sigma \epsilon, \quad \epsilon \sim \mathcal{N}(0, \mathcal{I})$, where \( \sigma \in [\sigma_{\min}, \sigma_{\max}] \) denotes the noise level. The model \( \epsilon_\theta(x_\sigma, \sigma) \) is trained to predict \( \epsilon \) using a denoising loss. Following the EDM formulation of \citet{karras2022elucidatingdesignspacediffusionbased}, the training loss is expressed as follows: \begin{align}
\mathbb{E}_{x_0, \sigma, \epsilon} \left[
\lambda(\sigma) \cdot \left\| \epsilon_\theta(x_0 + \sigma \epsilon, \sigma) - \epsilon \right\|^2
\right],
\label{eq:edm-true-loss}
\end{align}
where the EDM weighting function $\lambda(\sigma) = \frac{\sigma^2 + \sigma_{\text{data}}^2}{(\sigma \cdot \sigma_{\text{data}})^2}$ reweights contributions across noise levels. In the DP setting, Equation~\ref{eq:edm-true-loss} is approximated by sampling \( K \) independent noise levels \( \{\sigma_k\}_{k=1}^K \) and corresponding \( \epsilon_k \) for each training example, resulting in a Monte Carlo approximation. This formulation, referred to as the \emph{noise multiplicity} approach by \citet{dockhorn2023differentially}, takes the form:
\begin{align}
 \frac{1}{K} \sum_{k=1}^K \lambda(\sigma_k) \cdot \left\| \epsilon_\theta(x_0 + \sigma_k \epsilon_k, \sigma_k) - \epsilon_k \right\|^2,
\label{eq:edm-loss}
\end{align}
Each training example contributes $K$ noise-conditioned loss evaluations per step. The per-sample gradients are averaged to reduce variance which must be accounted for in DP-SGD.

\subsection{Membership Inference Attacks}
\label{subsec:MIAprelims}

Membership inference attacks (MIAs) aim to determine whether a particular data point was used to train a machine learning model \cite{shokri2017membership}. Given a model \( f \) trained on a dataset \( D \) and a sample \( x \), the adversary \( \mathcal{A} \) attempts to infer whether \( x \in D \) (a \emph{member}) or \( x \notin D \) (a \emph{non-member}). Formally, the attack is framed as a binary decision function:
\begin{equation}
\mathcal{A}: x \mapsto \mathcal{C}(s_f(x)) \in \{0, 1\},
\end{equation}
where $\mathcal{C}$ is a classifier, \( s_f(x) \in \mathbb{R} \) is a scalar \emph{attack score} extracted from the model's behavior on input \( x \). The attack score, which is model-specific, quantifies the model’s confidence or sensitivity on input \( x \) and is used by the attacker to infer membership via a classifier $\mathcal{C}$. In GAN-based attacks, it is often computed from the discriminator’s raw logit, e.g. $ s_f(x) = D_\psi(x)$ \cite{Chen_2020}, which reflects the discriminator’s confidence that \( x \) is real. 

In diffusion models, the attack score is typically the scalar denoising loss: $s_f(x) = \mathbb{E}_{\epsilon, \sigma} \left\| \epsilon_\theta(x + \sigma \epsilon, \sigma) - \epsilon \right\|^2$ \cite{carlini2023extractingtrainingdatadiffusion}, 
which measures how well the model reconstructs noisy versions of \( x \). In both cases, members tend to have different scores (higher confidence or lower reconstruction error), enabling the attacker to distinguish them from non-members.

To quantify the effectiveness of such MIAs, we use the following definition of attacker advantage.
\begin{definition}[Attacker advantage, \citet{yeom2018privacyriskmachinelearning}] 
\label{advmia}
The attacker advantage quantifies the gap between true and false positive rates in membership inference:
\begin{equation} \label{advmiaeq}
     \ADV_\text{MIA} = \mathbb{P}[\mathcal{A}(x)=1 \mid x \in D] - \mathbb{P}[\mathcal{A}(x)=1 \mid x \notin D],
\end{equation}
where \( \mathcal{A}(x) \) is the attacker's decision on whether \( x \) is in the training set. A value of \( \ADV_\text{MIA} = 0 \) indicates perfect privacy: the attacker performs no better than random guessing. Higher values reflect greater privacy leakage, as the attacker can better distinguish training from non-training samples.
\end{definition}

A key factor enabling MIAs is the behavioral gap between training and unseen data \citep{shokri2017membership, yeom2018privacyriskmachinelearning}. A standard approach to exploit this gap is \emph{shadow modeling}, where the adversary trains auxiliary models on disjoint datasets with known membership labels to mimic the target model's behavior. These shadow models generate scores used to train a membership classifier \( \mathcal{C} \) that learns to distinguish members from non-members. Originally introduced in the black-box setting \citep{shokri2017membership}, shadow modeling has since been adapted to scenarios where the attacker has partial knowledge of the target’s architecture or training procedure \citep{Chen_2020, 8835245}.

To better understand what drives membership leakage, we now turn to the notion of algorithmic stability.
\subsection{Algorithmic Stability}
\label{subsec:stability}
Algorithmic stability measures how much a learning algorithm's output changes in response to small perturbations in the training data. It is a classical tool for understanding generalization \citep{bousquet2002stability,hardt2016trainfastergeneralizebetter}, and more recently, it has emerged as a key concept in quantifying privacy leakage.

In the context of membership inference, the link is intui\-tive: if a membership inference attack succeeds on a model, then its predictions must change noticeably between training and unseen examples. This suggests that small changes to the training data, like removing a single example, can influence the model's output. In contrast, a model whose predictions remain consistent when the data is slightly modified is more likely to resist such attacks.

To analyze the stability of a learning algorithm that maps a dataset to a function $f$ (which may represent the full model or a specific component trained by the algorithm), we require a metric to evaluate the quality of the function’s output. We therefore define a loss function $\ell(f,z) \in \mathbb{R}$, where $z$ denotes a data sample. Depending on the setting, $z$ can be either $z=(x,y)$, in which case $\ell(f,z) = c(f(x),y)$, or $z=(x)$, 
in which case $\ell(f,z) = c(f(x),x)$, where $c(\cdot,\cdot)$ is a cost function.

\begin{definition}[Uniform stability, \citet{bousquet2002stability}]
\label{def:stability}

The function \( f \) is \emph{\( \beta \)-uniformly stable} with respect to a loss function \( \ell \) if, for any training set \( D \) of size \( m \), and any index \( i \),
\begin{equation}
\left\| \ell(f_D, \cdot) - \ell(f_{D^{\setminus i}}, \cdot) \right\|_\infty \leq \beta.
\end{equation}
where $f_D$ and $f_{D^{\setminus i}}$  are hypothesis functions obtained
by training the algorithm with respectively dataset $D$ and 
$D^{\setminus i}.$ Uniform stability quantifies how sensitive a learning algorithm is to changes in a single training point. 

\end{definition}

Recent work has formalized the link between stability and privacy attacks. \citet{yeom2018privacyriskmachinelearning} show that high empirical advantage in a membership inference attack implies instability of the learning algorithm, and vice versa. Moreover, \citet{carlini9833649} further argue that instability is often exacerbated in overparameterized or poorly regularized models (conditions that commonly arise in generative modeling). These insights motivate the use of stability analysis as a tool for explaining privacy leakage.

In this work, we use uniform stability to compare the privacy properties of differentially private GANs and diffusion models. By analyzing how their outputs react to the removal of a single training point, we derive theoretical bounds on membership advantage.

\section{Stability–Based Analysis of MIA Risk in DP GANs and DP Diffusion Models}
\label{sec:theory}
Empirically, non-private GANs leak less membership information than diffusion models~\citep{carlini2023extractingtrainingdatadiffusion}. We provide a formal theoretical explanation grounded in \emph{uniform stability}, comparing the GANs and Diffusion models in the private setting.

We proceed in three steps. First, we show that the attack scores are Lipschitz with respect to the training loss (Props.~\ref{gl},~\ref{dl}), implying that loss stability transfers to score stability (Lemma~\ref{lem:score-stability}). Second, under a bounded score density, score stability bounds the membership advantage of any threshold attack (Thm.~\ref{thm:adv-bound}). Third, we derive a general DP-SGD stability bound (Lemma~\ref{lemma:dp-sgd-stability}) and instantiate it for GANs and diffusion (Lemmas~\ref{lemma:dp-gan-stability},~\ref{lemma:dp-diff-stability}) to compare privacy leakage between both models.

\subsection{Linking Uniform Stability to Attack Scores}

Uniform stability (Definition~\ref{def:stability}) bounds the \emph{loss} drift when a single training point is removed. To translate this into a bound on any \emph{attack score}, we introduce a regularity property connecting the training loss and the attack score. This property enables us to formally relate the algorithm’s training stability to the success of an MIA.

\begin{property}[Loss–score Lipschitz link for GANs]
\label{gl}
Let \( f = D_\psi \in \mathcal{F}_{\mathrm{GAN}} \) be a discriminator parameterized by \( \psi \), trained using the logistic loss. For any input \( x \in \mathcal{X} \) and label \( y \in \{-1, +1\} \), define:
\begin{itemize}
    \item The \emph{score} used by the attacker is the raw logit: \( s_f(x) := D_\psi(x) \).
    \item The \emph{training loss} is the logistic loss: \( \ell(f, x, y) := \log(1 + e^{-y\,s_f(x)}) \).
\end{itemize}

Assume the loss values lie in a compact interval \( [a, b] \subset \mathbb{R}_{>0} \). Then the map \( f \mapsto s_f(x) \) is Lipschitz with respect to \( \ell(f, x, y) \), with
\begin{equation}
|s_f(x) - s_{f'}(x)| \le L_s \cdot |\ell(f, x, y) - \ell(f', x, y)|,
\end{equation}
where \( L_s = \sup_{u \in [a, b]} \frac{e^u}{e^u - 1} \).
\end{property}

\begin{proof}
See Appendix~\ref{propertiesproofs}.
\end{proof}

\begin{property}[Loss–score Lipschitz link for diffusion models]
\label{dl}

Let \( f = \epsilon_\theta \in \mathcal{F}_{\mathrm{Diff}} \) be a denoising network parameterized by \( \theta \), trained using the EDM objective (Equation \ref{edmobjj}) \cite{karras2022elucidatingdesignspacediffusionbased}. Define:
\begin{itemize}
    \item the \emph{attack score} as the scalar denoising error:
    \begin{equation}
    s_f(x) := \mathbb{E}_{\epsilon, \sigma} \left\| \epsilon_\theta(x + \sigma \epsilon, \sigma) - \epsilon \right\|^2;
    \end{equation}
    \item the \emph{training loss} as the noise-weighted EDM objective:
    \begin{equation}
    \label{edmobjj}
    \ell(f, x) := \mathbb{E}_{\epsilon, \sigma} \left[ \lambda(\sigma) \cdot \left\| \epsilon_\theta(x + \sigma \epsilon, \sigma) - \epsilon \right\|^2 \right],
    \end{equation}
    where \( \lambda(\sigma) \in [\lambda_{\min}, \lambda_{\max}] \subset (0, \infty) \) is a bounded weighting function.
\end{itemize}
Then, for any \( f, f' \in \mathcal{F}_{\mathrm{Diff}} \) and any \( x \in \mathcal{X} \), the following inequality holds:
\begin{equation}
|s_f(x) - s_{f'}(x)| \le \frac{1}{\lambda_{\min}} \cdot \left\| \ell(f, \cdot) - \ell(f', \cdot) \right\|_\infty.
\end{equation}
That is, the attack score is \( \lambda_{\min}^{-1} \)-Lipschitz with respect to the training loss.
\end{property}

\begin{proof}
See in Appendix \ref{propertiesproofs}.
\end{proof}

\medskip
The following lemma generalizes the score–loss relationships from Properties~\ref{gl} and~\ref{dl}, yielding a stability bound on attack scores from the uniform stability of the training loss.

\begin{lemma}[Score stability]
\label{lem:score-stability}
If the learning algorithm is $\beta$-uniformly stable with respect to the loss $\ell$, and the attack score function satisfies Properties \ref{gl},\ref{dl}, then for all $x \in \mathcal{X}$:
\begin{equation}
|s_{f_D}(x) - s_{f_{D^{\setminus i}}}(x)| \le L_s \cdot \beta.
\end{equation}
\end{lemma}

\begin{proof}
Immediate by Lipschitz continuity of \( s_{\cdot}(x) \) from Properties~\ref{gl},\ref{dl}.
\end{proof}
\subsection{Stability Bound on Membership Advantage}
Uniform stability limits how much a model's behavior can change when a single training point is removed, making it harder for an adversary to distinguish members from non-members. While \citet{yeom2018privacyriskmachinelearning} showed that uniform stability bounds the membership advantage of threshold attacks based directly on the loss, our result extends this guarantee to a broader class of attacks. Specifically, we prove in Theorem \ref{thm:adv-bound} that any threshold-based adversary using a score function that is Lipschitz-continuous with respect to the loss, such as discriminator logits or denoising errors, also yields bounded membership advantage. This provides a new theoretical guarantee that captures more realistic attack settings beyond loss-based inference.

\begin{theorem}[Bound on membership advantage under uniform score stability]
\label{thm:adv-bound}
Let \( f \) be a learning algorithm that is \( \beta \)-uniformly stable with respect to a loss function \( \ell \), and suppose the loss–score Lipschitz condition holds with constant \( L_s > 0 \) (Lemma \ref{lem:score-stability}). Assume further that the distribution of the score \( s_{f_D}(x) \) admits a bounded density with upper bound $Q$. Then, for any threshold-based adversary of the form
\begin{equation}
\mathcal{A}(x) = \mathbb{I}\{s_{f_D}( x) \leq \tau\},
\end{equation}
the membership advantage is bounded as
\begin{equation}
\ADV_{\MIA} \;\leq\; 2 Q L_s \beta.
\end{equation}
\end{theorem}

\begin{proof}
    See Appendix \ref{proofth1}
\end{proof}

The bound in Theorem~\ref{thm:adv-bound} is informative only when \(2 Q L_s \beta < 1\), since by definition \(\ADV_{\MIA} \in [0, 1]\). This condition imposes a constraint on \(L_s \beta\). In particular, \(\beta\) decreases with the dataset size \(m\); for example, standard bounds for DP-SGD with per-sample gradient clipping yield \(\beta = \mathcal{O}(1/m)\), making the bound tighter as \(m\) increases (see Lemma~\ref{lemma:dp-sgd-stability} for a formal derivation of this bound). Conversely, \(L_s\) quantifies the sensitivity of the attack score to changes in the loss, and is specific to the model and chosen score function. Overall, the tightness of the bound reflects a trade-off between algorithmic stability and the score’s sensitivity to perturbations. We provide a more detailed discussion in Appendix~\ref{app:th1tradeoff}.

\subsection{Uniform Stability of Functions Trained by DP-SGD}
To understand the behaviour of the MIA advantage bound from Theorem~\ref{thm:adv-bound}, it is crucial to characterize the uniform stability parameter \(\beta\). In particular, our analysis relies on the fact that \(\beta = \mathcal{O}(1/m)\), a property we now formalize. We derive a general upper bound on the expected uniform stability of functions trained by DP-SGD, which we later instantiate for GAN discriminators and diffusion denoisers.

\begin{lemma}[Uniform stability of functions trained by DP-SGD]

\label{lemma:dp-sgd-stability}
Let \(\ell(f_\theta, z) \) be a loss that is \( L \)-Lipschitz in the parameters \( \theta \), for all \( z \in \mathcal{X} \times \mathcal{Y} \). Suppose DP-SGD runs for \( T \) steps with per-sample gradient clipping at norm \( C \). At each step \( t \), a mini-batch \( B_t \) of constant size \( b \) is sampled uniformly without replacement from a dataset of size \( m \), and a learning rate \( \alpha_t \) is applied. Then, uniform stability of functions trained by DP-SGD is defined as follows:

\begin{equation}
\label{expuniformstab}
\beta \;:=\; \sup_{z,\,i}\; \mathbb{E}\!\left[\,\big| \ell(f_D, z) - \ell(f_{D^{\setminus i}}, z) \big|\,\right]
\;\le\; \frac{2 L C}{m}\,\sum_{t=1}^T \alpha_t.
\end{equation}
In particular, for constant step size \( \alpha_t = \alpha \), we have:
\begin{equation}
\beta \le \frac{2 L C \alpha T}{m}.
\end{equation}
\end{lemma}

Notice that we consider two neighboring datasets \(D\) and \(D^{\setminus i}\), differing in a single example, and analyze two executions of DP-SGD that are coupled via shared randomness, that is, they use the same sequence of mini-batches and the same Gaussian noise vectors. This coupling isolates the effect of the data perturbation from that of the stochastic noise. Consequently, the bound reflects the sensitivity of the algorithm rather than the effect of noise, which is why the DP noise scale \(\sigma\) does not appear explicitly in the stability bound (More details in Appendix \ref{cr}).

\begin{proof}
Let $\theta_t$ and $\theta_t'$ denote the parameter vectors at step $t$ of two models trained with DP-SGD on $D$ and $D^{\setminus i}$, respectively. At each step,  we sample a mini-batch $B_t$ of size $b$ and perform the update:
\begin{equation}
\theta_{t+1} \leftarrow \theta_t - \alpha_t \left( \frac{1}{|B_t|} \sum_{j \in B_t} \text{clip}(\nabla \ell(f_{\theta_t}; z_j), C) + \eta_t \right),
\end{equation}
where $\eta_t \sim \mathcal{N}(0, \sigma^2 I)$, and similarly for $\theta_t'$. 

At step $t$, the update uses the \emph{mean} of clipped per-sample gradients. If $z_i$ is not in the mini-batch, the parameter updates for $\theta_t$ and $\theta_t'$ coincide; if it is, the mean changes by at most $C/b$ in norm because clipping ensures each per-sample contribution has norm $\le C$. The event “$z_i$ is in the batch” occurs with probability $b/m$,  when the differing sample is included in the mini-batch. Hence the \emph{expected} change in the update (in norm) at step $t$ is at most $(b/m)\cdot (C/b)=C/m$. Multiplying by the step size $\alpha_t$, we get
\begin{equation}
\mathbb{E}\big[\|\theta_{t+1}-\theta_{t+1}'\|\big]
\;\le\;
\mathbb{E}\big[\|\theta_t-\theta_t'\|\big] \;+\; \alpha_t\,\frac{C}{m}.
\end{equation}
Therefore, it yields that
\(
\mathbb{E}\big[\|\theta_T-\theta_T'\|\big] \le \frac{C}{m}\sum_{t=1}^T \alpha_t.
\)
Since $\ell(\cdot;z)$ is $L$-Lipschitz in $\theta$, we conclude
\begin{align}
\mathbb{E}\!\left[\,\big|\ell(f_D,z)-\ell(f_{D^{\setminus i}},z)\big|\,\right]
\;&\le\; L\,\mathbb{E}\big[\|\theta_T-\theta_T'\|\big]\\ 
\;&\le\; \frac{L C}{m}\sum_{t=1}^T \alpha_t.
\end{align}
\end{proof}

\subsection{Model-Specific Stability Bounds}
We apply Lemma~\ref{lemma:dp-sgd-stability} to the two generative families studied here. Our bounds have the same structure and differ through the loss Lipschitz constants and total update counts.

\begin{lemma}[Stability bound for DP-GANs]
\label{lemma:dp-gan-stability}
Let $D_\psi$ denote the discriminator of a GAN, trained with DP-SGD over $T_G$ steps with per-sample clipping at norm $C$ and learning rates $\{\alpha_t\}_{t=1}^{T_G}$. Assume the discriminator score $s_\psi(x) := D_\psi(x)$ is $L$-Lipschitz in parameters $\psi$, and define the logistic loss
\begin{equation}
\ell_G(\psi; x, y) := \log\!\big(1 + e^{-y\,s_\psi(x)}\big), \quad y \in \{-1, +1\}.
\end{equation}
Then the expected uniform stability of the DP-GAN discriminator satisfies
\begin{equation}
\beta_{\mathrm{GAN}} \;\le\; \frac{2 L C}{m} \sum_{t=1}^{T_G} \alpha_t.
\end{equation}
\end{lemma}

\begin{proof}
The function $z \mapsto \log(1 + e^{-y z})$ is $1$-Lipschitz in $z$ for any $y \in \{-1, +1\}$. If $s_\psi(x)$ is $L$-Lipschitz in $\psi$, then by the composition of Lipschitz functions, the loss $\ell_G(\psi; x, y)$ is $L$-Lipschitz in $\psi$. That is,
\begin{equation}
|\ell_G(\psi; x, y) - \ell_G(\psi'; x, y)| \le L\,\|\psi - \psi'\|.
\end{equation}
Applying Lemma~\ref{lemma:dp-sgd-stability} with Lipschitz constant $L_G \le L$, dataset size $m$, and total number of steps $T = T_G$, the uniform stability satisfies
\begin{equation}
\beta_{\mathrm{GAN}} \;\le\; \frac{2 L_G C}{m} \sum_{t=1}^{T_G} \alpha_t \;\le\; \frac{2 L C}{m} \sum_{t=1}^{T_G} \alpha_t.
\end{equation}

\end{proof}

\begin{lemma}[Stability bound for DP-diffusion models]
\label{lemma:dp-diff-stability}
Let $\epsilon_\theta$ be a denoiser trained with DP-SGD over $T_D$ steps using the multi-pass EDM loss:
\begin{equation}
\ell_D(\theta; x, y) := \frac{1}{K} \sum_{k=1}^K \lambda(\sigma_k)\, \big\| \epsilon_\theta(x + \sigma_k \epsilon_k, \sigma_k, y) - \epsilon_k \big\|^2,
\end{equation}
where $\epsilon_k \sim \mathcal{N}(0, I)$ and $\lambda(\sigma_k) := \frac{\sigma_k^2 + \sigma_{\mathrm{data}}^2}{(\sigma_k \sigma_{\mathrm{data}})^2}$ are fixed weights. Assume that the prediction error is uniformly bounded: $\| \epsilon_\theta(x + \sigma_k \epsilon_k, \sigma_k, y) - \epsilon_k \| \le B$ for all $k$ and all $\theta$, and that the denoiser $\epsilon_\theta$ is $L$-Lipschitz in $\theta$ for fixed input.
Then the per-sample training loss $\ell_D(\theta; x, y)$ is Lipschitz in $\theta$ with constant
\begin{equation}
L_D \le 2\, \bar\lambda\, L\, B, \:\:\: \text{where } \bar\lambda := \frac{1}{K} \sum_{k=1}^K \lambda(\sigma_k).
\end{equation}
Consequently, the expected uniform stability of the DP-diffusion model satisfies
\begin{equation}
\beta_{\mathrm{Diff}} \le \frac{2 L_D C}{m} \sum_{t=1}^{T_D} \alpha_t
\;\le\; \frac{4\, \bar\lambda\, L\, B\, C}{m} \sum_{t=1}^{T_D} \alpha_t.
\end{equation}

\end{lemma}

\begin{proof}
Let \( f_\theta := \epsilon_\theta(x + \sigma_k \epsilon_k, \sigma_k, y) \). For each $k$, consider the weighted term in the loss,
$\tau_k(\theta) := \lambda(\sigma_k)\,\|f_\theta - \epsilon_k\|^2$. 
Assume the uniform error bound 
\(\|f_\theta-\epsilon_k\|\le B\) for all \(\theta\) and that
\(f_\theta\) is \(L\)-Lipschitz in \(\theta\) , i.e.,
\(\|f_\theta-f_{\theta'}\|\le L\|\theta-\theta'\|\). Then
\begin{align}
&\big|\tau_k(\theta)-\tau_k(\theta')\big|
= \lambda(\sigma_k)\,\big|\,\|f_\theta-\epsilon_k\|^2 - \|f_{\theta'}-\epsilon_k\|^2\,\big| \\
&= \lambda(\sigma_k)\,\big|\,\langle (f_\theta-\epsilon_k)+(f_{\theta'}-\epsilon_k),\, (f_\theta-\epsilon_k)-(f_{\theta'}-\epsilon_k)\rangle\,\big| \\
&= \lambda(\sigma_k)\,\big|\,\langle f_\theta+f_{\theta'}-2\epsilon_k,\; f_\theta-f_{\theta'}\rangle\,\big| \\
& \label{csti} \le \lambda(\sigma_k)\,(\|f_\theta-\epsilon_k\|+\|f_{\theta'}-\epsilon_k\|)\,\|f_\theta-f_{\theta'}\| \\
&\le 2\,\lambda(\sigma_k)\,B\,\|f_\theta-f_{\theta'}\|
\quad\text{(uniform error bound)} \\
&\le 2\,\lambda(\sigma_k)\,B\,L\,\|\theta-\theta'\|
\quad\text{(Lipschitz continuity of }f_\theta\text{ in }\theta).
\end{align}
Here, Eq.~\eqref{csti} follows from the Cauchy--Schwarz and the triangle inequality. Therefore, each term in the sum is $2 \lambda(\sigma_k) L B$-Lipschitz in $\theta$, and the average loss over $k=1,\dots,K$ is:
\begin{align}
| \ell_D(\theta; x, y) - \ell_D(\theta'; x, y) | &\le \frac{1}{K} \sum_{k=1}^K 2 \lambda(\sigma_k) L B \| \theta - \theta' \| \\
&= 2\, \bar\lambda\, L\, B \| \theta - \theta' \|.
\end{align}
Applying Lemma~\ref{lemma:dp-sgd-stability} with Lipschitz constant $L_D = 2 \bar\lambda L B$ and total updates $T_D$ yields:
\begin{equation}
\beta_{\mathrm{Diff}} \le \frac{2 L_D C}{m} \sum_{t=1}^{T_D} \alpha_t
= \frac{4 \bar\lambda L B C}{m} \sum_{t=1}^{T_D} \alpha_t.
\end{equation}
\end{proof}

\paragraph{Why DP-diffusion yields higher membership leakage.} The stability bound from Lemma~\ref{lemma:dp-sgd-stability} scales with the product of the loss Lipschitz constant and the total number of DP-SGD steps. For DP-GANs, only the discriminator is trained with a logistic loss that is $L$-Lipschitz in parameters. In contrast, diffusion models are trained with a weighted multi-pass EDM loss, where each term is scaled by $\lambda(\sigma_k) = \frac{\sigma_k^2 + \sigma_{\text{data}}^2}{(\sigma_k \sigma_{\text{data}})^2}$. These weights increase rapidly as $\sigma_k$ decreases, amplifying the influence of low-noise terms and leading to a large effective Lipschitz constant. Under a shared network smoothness $L$, we have $L_G \le L$ and $L_D \le 2 \bar\lambda L B$, where $\bar\lambda = \frac{1}{K} \sum_k \lambda(\sigma_k)$ is typically large. Moreover, diffusion models are typically trained for more steps than GAN discriminators ($T_D \gg T_G$). Together, these factors imply
\begin{equation}
\beta_{\mathrm{Diff}} \gg \beta_{\mathrm{GAN}}.
\end{equation}
Applying Theorem~\ref{thm:adv-bound} with a score function of bounded density $Q$ yields $\ADV_{\MIA}^{\GAN} \le 2 Q L_G \beta_{\GAN}$, and $\ADV_{\MIA}^{\Diff} \le 2 Q L_D \beta_{\Diff}$. Since both $L_D \gg L_G$ and $\sum_{t=1}^{T_D} \alpha_t \gg \sum_{t=1}^{T_G} \alpha_t$, the upper bound on $\ADV_{\MIA}^{\Diff}$ is significantly larger than that of $ \ADV_{\MIA}^{\GAN}$, providing a theoretical explanation of greater membership leakage for DP-diffusion models.

\section{Empirical Analysis of Membership Leakage in DP-GAN and DP-Diffusion}

Building on our theoretical analysis showing that GANs offer greater robustness than diffusion models against membership inference under differential privacy, we empirically assess the extent of membership leakage across both training architectures and a range of privacy budgets. This allows us to assess whether differential privacy mitigates architectural disparities or whether distinct privacy risks remain.

\subsection{Experimental Setup}

Our study compares GANs and diffusion models trained both with and without differential privacy using \texttt{Opacus} \citep{opacus}. We vary the privacy budget $\varepsilon \in \{\infty, 10, 5, 1\}$, fix $\delta = 10^{-5}$, and apply DP-SGD with per-sample gradient clipping and additive Gaussian noise. Privacy spending is tracked using the Moments Accountant.

All experiments are conducted on the MNIST dataset \cite{ lecun-mnisthandwrittendigit-2010}. To ensure fair comparison, all models share the same optimization settings and training budget. We evaluate sample quality using the Fréchet Inception Distance (FID) \cite{heusel2018ganstrainedtimescaleupdate}, and assess membership inference vulnerability using standard metrics: accuracy, precision, true positive rate (TPR), false positive rate (FPR), and area under the ROC curve (AUC). Formal definitions of these metrics and implementation details are provided in Appendix~\ref{expeap}.

\begin{table*}
\centering
{\fontsize{10}{12}\selectfont
\begin{tabular}{lccccccccccccccc}
\toprule
\multirow{2}{*}{$\varepsilon$} 
& \multicolumn{5}{c}{GAN} 
& \multicolumn{5}{c}{GAN ADP} 
& \multicolumn{5}{c}{DM} \\
\cmidrule(lr){2-6} \cmidrule(lr){7-11} \cmidrule(lr){12-16}
& Acc & Prec & TPR & FPR & AUC 
& Acc & Prec & TPR & FPR & AUC 
& Acc & Prec & TPR & FPR & AUC \\
\midrule
$\infty$ 
& 0.74 & 0.69 & 0.88 & 0.39 & 0.74
&  x   &  x   &  x   &  x   &  x   
& 0.79 & 0.72 & 0.92 & 0.35 & 0.75 \\
10       
& 0.50 & 0.50 & 0.57 & 0.56 & 0.53
& 0.51 & 0.51 & 0.33 & 0.32 & 0.51
&  0.58 & 0.58 & 0.56 & 0.40 & 0.60 \\
5        
& 0.50 & 0.50 & 0.69 & 0.68 & 0.50
& 0.50 & 0.50 & 0.69 & 0.68 & 0.49
&  0.57 & 0.58 & 0.49 & 0.35 & 0.55 \\
1        
& 0.51 & 0.62 & 0.05 & 0.03 & 0.49
& 0.51 & 0.53 & 0.08 & 0.07 & 0.48
&  0.55 & 0.58 & 0.36 & 0.26 & 0.52 \\
\bottomrule
\end{tabular}
}
\vspace{0.5em}
\caption{Attack scores for GAN (fixed-step regime), GAN ADP (adaptative regime), and diffusion models (DM) on MNIST across privacy levels $\varepsilon$.}
\label{tab:attack_scores}
\end{table*}

\begin{table}
\setlength{\tabcolsep}{3pt} 
\centering
{\fontsize{10}{12}\selectfont 
\begin{tabular}{lccc}
\toprule
$\varepsilon$ & GAN & GAN ADP & DM \\
\midrule
$\infty$ & 5.1   & X     & 3.1 \\
10       & 39.7  & 18.6  & 14.1 \\
5        & 101.8 & 34.1  & 30.2 \\
1        & 183.2 & 73.2  & 72.9 \\
\bottomrule
\end{tabular}
}
\vspace{0.5em}
\caption{Mean FID scores for GAN, GAN ADP, and DM on MNIST across privacy levels $\varepsilon$ over 5 generation runs, computed on test samples.}
\label{tab:fid_only}
\end{table}

\paragraph{DP-GAN.} We implement class-conditional DP-GANs following the architecture of \citet{bie2023privategansrevisited}. Only the discriminator is trained with DP-SGD, while the generator is updated non-privately. This setup satisfies differential privacy for the entire pipeline via the post-processing property. 

Training DP-GANs can be unstable, as noise injected into the discriminator degrades gradient quality. To mitigate this, we adopt two balancing strategies for generator and discriminator updates, both proposed by \citet{bie2023privategansrevisited}. In the \emph{fixed-step} regime, we perform a fixed number \( n_D \) of discriminator updates per generator update; increasing \( n_D \) typically improves stability and sample quality, especially at lower privacy budgets. In the \emph{adaptive} regime, \( n_D \) is dynamically adjusted based on the discriminator’s accuracy on fake samples, enabling more flexible training schedules.

\paragraph{DP-Diffusion.}
Our diffusion models follow the DPDM framework of \citet{dockhorn2023differentially}. For each training example, we apply the noise multiplicity loss, which averages denoising losses over \( K = 32 \) independently sampled noise levels, as formalized in Eq.~\ref{eq:edm-loss}. Each noise scale \( \sigma_k \) is drawn independently from a log-normal distribution: \( \sigma_k \sim \mathrm{LogNormal}(p_{\text{mean}}, p_{\text{std}}) \cdot \sigma_{\text{data}} \). We use the same loss formulation during both training and membership inference to ensure consistency in score computation.

\subsection{Three-Stage Membership Inference Pipeline}
To rigorously assess vulnerability to membership inference, we follow a standardized three-stage pipeline \citep{shokri2017membership, carlini9833649}, adapting it to generative models trained under differential privacy.

We begin by splitting the MNIST dataset into two disjoint halves: one for the target model (private data) and the other for shadow models (public data). Within each half, we further split the data to define member and non-member sets used during the attack.

\textbf{Stage 1: Target model training.} We train a differentially private target generative model on the private subset using DP-SGD. A portion of this data is used for training (members), while the remainder is held out (non-members) for evaluation. This model serves as the attack target.

\textbf{Stage 2: Shadow model training.} We train 20 shadow models using the same architecture and training protocol as the target model. Each shadow is trained on a distinct random split of the public subset. For each shadow model, we define members as the training portion and non-members as the held-out portion. This process is repeated independently per shadow model to diversify the attack training data.

\textbf{Stage 3: Attack.} After training, we compute per-sample scores on all shadow data and use them to train an attack model that distinguishes members from non-members. For GANs, we use the raw logits from the discriminator on member and non-member samples.At test time, we apply the trained attack model to the target by computing scores on held-out samples and predicting membership using the trained classifier. For diffusion models, we compute scalar denoising losses under fixed noise conditions following the strong Likelihood Ratio Attack (LiRA) of \citet{carlini2023extractingtrainingdatadiffusion}, which estimates membership by comparing the likelihood of each loss under member and non-member distributions and applying a likelihood–ratio threshold.

The attacks are black-box: the adversary has no access to the target model’s parameters, gradients, or training data. However, we assume knowledge of the architecture, training procedure, and privacy parameters. This yields a practical, generalizable attack strategy without requiring handcrafted decision rules.

\subsection{Experimental Results}

To evaluate generative quality, we report the FID of the target models on the MNIST dataset in Table~\ref{tab:fid_only}, averaged over five independent training runs per privacy level~$\varepsilon$. To assess membership leakage, we report attack performance across privacy levels in Table~\ref{tab:attack_scores}, using accuracy, precision, true positive rate (TPR), false positive rate (FPR), and area under the ROC curve (AUC). 

Table~\ref{tab:attack_scores} confirms that GANs are more robust to membership inference attacks under differential privacy, with leakage degrading sharply and stabilizing near random for moderate privacy budgets ($\varepsilon \le 10$). In contrast, diffusion models degrade more gradually and retain non-trivial leakage even at $\varepsilon = 1$. These empirical trends support our theoretical stability analysis: the weighted multi-pass denoising objective in diffusion models amplifies sensitivity to individual training samples, resulting in lower stability and increased privacy risk. While adaptive GANs (ADP) substantially improve FID compared to standard GANs, their vulnerability to membership inference remains similar, indicating that higher sample quality does not necessarily leads to stronger privacy guarantees.

Beyond validating the theory, our experiments provide the first systematic evaluation of membership leakage in differentially private generative models. To our knowledge, prior work introducing DP-GANs and DP-diffusion models has not assessed their vulnerability to MIA.

\section{Conclusion}

In this paper, we presented the first unified theoretical and empirical study of membership inference risk in differentially private generative models. Our analysis formalizes the connection between uniform stability and adversarial advantage, showing that the training architecture directly impacts the extent of membership leakage. In particular, we show that Diffusion models are more susceptible to membership inference than GANs under equivalent privacy budgets, due to their training dynamics. These findings highlight that evaluating generative models under DP requires more than tracking privacy parameters alone. We hope this work motivates further research of training-induced vulnerabilities in private learning systems.

\section*{Acknowledgments}

This project was provided with computing AI and storage resources by GENCI at IDRIS thanks to the grant 2024-A0171015707 on the supercomputer Jean Zay's V100/A100/H100 partitions.

\bibliography{aaai2026}

\appendix

\onecolumn

\setcounter{secnumdepth}{2}

\section{Proofs of Properties \ref{gl} and \ref{dl}}
\label{propertiesproofs}

\setcounter{property}{0}

\begin{property}[Loss–score Lipschitz link for GANs]
Let \( f = D_\psi \in \mathcal{F}_{\mathrm{GAN}} \) be a discriminator parameterized by \( \psi \), trained using the logistic loss. For any input \( x \in \mathcal{X} \) and label \( y \in \{-1, +1\} \), define:
\begin{itemize}
    \item The \emph{score} used by the attacker is the raw logit: \( s_f(x) := D_\psi(x) \).
    \item The \emph{training loss} is the logistic loss: \( \ell(f, x, y) := \log(1 + e^{-y\,s_f(x)}) \).
\end{itemize}

Assume the loss values lie in a compact interval \( [a, b] \subset \mathbb{R}_{>0} \). Then the map \( f \mapsto s_f(x) \) is Lipschitz with respect to \( \ell(f, x, y) \), with
\begin{equation}
|s_f(x) - s_{f'}(x)| \le L_s \cdot |\ell(f, x, y) - \ell(f', x, y)|,
\end{equation}
where \( L_s = \sup_{u \in [a, b]} \frac{e^u}{e^u - 1} \).
\end{property}

\begin{proof}
Let \( f = D_\psi \) and \( f' = D_{\psi'} \) be two discriminators in \( \mathcal{F}_{\mathrm{GAN}} \), and fix any input \( x \in \mathcal{X} \). Define the score function as the discriminator's logit output:
\begin{equation}
s_f(x) := D_\psi(x), \qquad s_{f'}(x) := D_{\psi'}(x).
\end{equation}
Define the training loss on real data as:
\begin{equation}
\ell(f, x, y) := \log(1 + e^{-y\,s_f(x)}, \ell(f', x, y) := \log(1 + e^{-y s_{f'}(x)}).
\end{equation}
Let \( g(t) := \log(1 + e^{-t}) \), so that \( \ell(f, x) = g(s_f(x)) \). Since \( g \) is strictly decreasing and smooth, it is invertible on \( \mathbb{R} \). Its inverse is given by:
\begin{equation}
g^{-1}(u) = -\log(e^u - 1) \quad \text{for } u > 0.
\end{equation}
For $g^{-1}$ continuous and differentiable on $u>0$. Hence,
\begin{align}
&y \cdot s_f(x) = g^{-1}(\ell(f, x, y)) \Rightarrow s_f(x) = y \cdot g^{-1}(\ell(f, x, y))
,\\ &s_{f'}(x) = y \cdot g^{-1}(\ell(f', x)).
\end{align}
By the mean value theorem applied to \( g^{-1} \), there exists \( \xi \in [\ell(f, x, y), \ell(f', x, y)] \) such that:
\begin{align}
|s_f(x) - s_{f'}(x)| &= |g^{-1}(\ell(f, x)) - g^{-1}(\ell(f', x))|\\
&= |(g^{-1})'(\xi)| \cdot |\ell(f, x, y) - \ell(f', x, y)|.
\end{align}

We compute the derivative of \( g^{-1} \):
\begin{equation}
(g^{-1})'(u) = -\frac{e^u}{e^u - 1}, \quad \text{so} \quad |(g^{-1})'(u)| = \frac{e^u}{e^u - 1}.
\end{equation}

Assume that the loss values \( \ell(f, x, y) \) and \( \ell(f', x, y) \) lie in a compact interval \( [a, b] \subset (0, \infty) \). Then the quantity \( \frac{e^u}{e^u - 1} \) is bounded on \( [a, b] \), since it is continuous on the compact interval \( [a, b] \subset (0, \infty) \), and we define:
\begin{equation}
L_s := \sup_{u \in [a,b]} \frac{e^u}{e^u - 1}.
\end{equation}

It follows that:
\begin{equation}
|s_f(x) - s_{f'}(x)| \le L_s \cdot |\ell(f, x, y) - \ell(f', x, y)|,
\end{equation}
\end{proof}

\begin{property}[Loss–score Lipschitz link for diffusion models]

Let \( f = \epsilon_\theta \in \mathcal{F}_{\mathrm{Diff}} \) be a denoising network parameterized by \( \theta \), trained using the EDM objective (Equation \ref{edmobjj}) \cite{karras2022elucidatingdesignspacediffusionbased}. Define:
\begin{itemize}
    \item the \emph{attack score} as the scalar denoising error:
    \begin{equation}
    s_f(x) := \mathbb{E}_{\epsilon, \sigma} \left\| \epsilon_\theta(x + \sigma \epsilon, \sigma) - \epsilon \right\|^2;
    \end{equation}
    \item the \emph{training loss} as the noise-weighted EDM objective:
    \begin{equation}
    \label{edmobjj}
    \ell(f, x) := \mathbb{E}_{\epsilon, \sigma} \left[ \lambda(\sigma) \cdot \left\| \epsilon_\theta(x + \sigma \epsilon, \sigma) - \epsilon \right\|^2 \right],
    \end{equation}
    where \( \lambda(\sigma) \in [\lambda_{\min}, \lambda_{\max}] \subset (0, \infty) \) is a bounded weighting function.
\end{itemize}
Then, for any \( f, f' \in \mathcal{F}_{\mathrm{Diff}} \) and any \( x \in \mathcal{X} \), the following inequality holds:
\begin{equation}
|s_f(x) - s_{f'}(x)| \le \frac{1}{\lambda_{\min}} \cdot \left\| \ell(f, \cdot) - \ell(f', \cdot) \right\|_\infty.
\end{equation}
That is, the attack score is \( \lambda_{\min}^{-1} \)-Lipschitz with respect to the training loss.
\end{property}

\begin{proof}
Let \( f = \epsilon_\theta \) and \( f' = \epsilon_{\theta'} \) be two denoising networks in \( \mathcal{F}_{\mathrm{Diff}} \), and fix an input \( x \in \mathcal{X} \).

We define the attack score as:
\begin{equation}
s_f(x) := \mathbb{E}_{\epsilon, \sigma} \left\| \epsilon_\theta(x + \sigma \epsilon, \sigma) - \epsilon \right\|^2,
\end{equation}
and the training loss as:
\begin{equation}
\ell(f, x) := \mathbb{E}_{\epsilon, \sigma} \left[ \lambda(\sigma) \cdot \left\| \epsilon_\theta(x + \sigma \epsilon, \sigma) - \epsilon \right\|^2 \right].
\end{equation}

Let us denote:
\begin{equation}
a_f(x, \epsilon, \sigma) := \left\| \epsilon_\theta(x + \sigma \epsilon, \sigma) - \epsilon \right\|^2.
\end{equation}

Then we can write:
\begin{equation}
s_f(x) = \mathbb{E}_{\epsilon, \sigma} [a_f(x, \epsilon, \sigma)], \quad
\ell(f, x) = \mathbb{E}_{\epsilon, \sigma} [\lambda(\sigma) \cdot a_f(x, \epsilon, \sigma)].
\end{equation}

Since \( \lambda(\sigma) \in [\lambda_{\min}, \lambda_{\max}] \subset (0, \infty) \), for all $x, \epsilon, \sigma$, we have:
\begin{equation}
\lambda_{\min} \cdot a_f(x, \epsilon, \sigma) \le \lambda(\sigma) \cdot a_f(x, \epsilon, \sigma) \le \lambda_{\max} \cdot a_f(x, \epsilon, \sigma)
\quad.
\end{equation}

Taking expectation over \( (\epsilon, \sigma) \), we get:
\begin{equation}
\lambda_{\min} \cdot \mathbb{E}_{\epsilon, \sigma}[a_f(x, \epsilon, \sigma)]
\leq \ell(f, x)
\leq \lambda_{\max} \cdot \mathbb{E}_{\epsilon, \sigma}[a_f(x, \epsilon, \sigma)].
\end{equation}

By definition of \( s_f(x) = \mathbb{E}_{\epsilon, \sigma}[a_f(x, \epsilon, \sigma)] \), this gives:
\begin{equation}
\lambda_{\min} \cdot s_f(x) \leq \ell(f, x) \leq \lambda_{\max} \cdot s_f(x).
\end{equation}

Therefore, dividing through and using the positivity of \( \lambda_{\min} \) and \( \lambda_{\max} \), we conclude, pointwise in \( x \),
\begin{equation}
\lambda_{\max}^{-1} \cdot \ell(f, x)
\leq s_f(x)
\leq \lambda_{\min}^{-1} \cdot \ell(f, x).
\end{equation}

Similarly, for the difference between two models:
\begin{align}
|s_f(x) - s_{f'}(x)|
&= \left| \mathbb{E}_{\epsilon, \sigma} \left[ a_f(x, \epsilon, \sigma) - a_{f'}(x, \epsilon, \sigma) \right] \right| \\
&\le \mathbb{E}_{\epsilon, \sigma} \left| a_f(x, \epsilon, \sigma) - a_{f'}(x, \epsilon, \sigma) \right| \\
&\le \lambda_{\min}^{-1} \cdot \mathbb{E}_{\epsilon, \sigma} \left| \lambda(\sigma) \cdot a_f(x, \epsilon, \sigma) - \lambda(\sigma) \cdot a_{f'}(x, \epsilon, \sigma) \right| \\
&= \lambda_{\min}^{-1} \cdot \left| \ell(f, x) - \ell(f', x) \right|.
\end{align}

Taking the supremum over \( x \in \mathcal{X} \), we obtain:
\begin{equation}
|s_f(x) - s_{f'}(x)| \le \lambda_{\min}^{-1} \cdot \| \ell(f, \cdot) - \ell(f', \cdot) \|_\infty.
\end{equation}

Thus, the property holds with Lipschitz constant \( L_s := \lambda_{\min}^{-1} \), completing the proof.
\end{proof}

\section{Proof of Theorem 1}
\label{proofth1}

\begin{proof}
Let \( D = \{x_1, \dots, x_m\} \sim \mathcal{P}^m \) be the training dataset, and let \( x \sim \mathcal{P} \) be an independent sample. Fix \( i \in \{1, \dots, m\} \), and let \( D^{\setminus i} = D \setminus \{x_i\} \) be the neighboring dataset obtained by removing \( x_i \).

By definition, the membership advantage of the adversary \( \mathcal{A} \) is
\begin{equation}
\ADV_{\MIA} = \left| \Pr[\mathcal{A}(x_i) = 1 \mid x_i \in D] - \Pr[\mathcal{A}(x) = 1 \mid x \notin D] \right|.
\end{equation}
We analyze this by comparing the adversary’s predictions on models trained on \( D \) and on \( D^{\setminus i} \). By Lemma~\ref{lem:score-stability}, the score function satisfies
\begin{equation}
|s_{f_D}(x) - s_{f_{D^{\setminus i}}}(x)| \leq L_s \beta
\quad \text{for all } x \in \mathcal{X}.
\end{equation}
This implies that the adversary’s predictions can differ only when the score is within \( L_s \beta \) of the threshold \( \tau \), that is:
\begin{align}
\mathbb{I}\{s_{f_D}(x) &\leq \tau\} \neq \mathbb{I}\{s_{f_{D^{\setminus i}}}(x) \leq \tau\}
\quad \\
&\Rightarrow \quad
s_{f_D}(x) \in (\tau - L_s \beta, \tau + L_s \beta).
\end{align}

Define the margin region \( M = (\tau - L_s \beta, \tau + L_s \beta) \). Suppose \( s_{f_D}(x) \) admits a probability density function \( p \) bounded above by \( Q \), i.e., \( p(u) \leq Q \) for all \( u \in \mathbb{R} \). Then the difference in prediction probabilities satisfies
\begin{equation}
\label{eq:doubleineth}
\left| \Pr[\mathcal{A}(x_i) = 1] - \Pr[\mathcal{A}(x) = 1] \right| \leq \Pr[s_{f_D}(x) \in M] \leq 2 Q L_s \beta.
\end{equation}
Indeed, since \( p(u) \leq Q \), the probability mass in the margin region \( M \) is at most:
\begin{equation}
\Pr[s_{f_D}(x) \in M] \leq \int_{\tau - L_s \beta}^{\tau + L_s \beta} p(u) \, du \leq 2 Q L_s \beta.
\end{equation}
More details on Equation~\ref{eq:doubleineth} are provided below
\end{proof}

\paragraph{Comments Eq. \ref{eq:doubleineth}.}
Let's give more details on how we obtained the following:
$$
\left| \Pr[\mathcal{A}(z_i) = 1] - \Pr[\mathcal{A}(z) = 1] \right| \leq \Pr[s_{f_D}(x) \in M] \leq 2 Q L_s \beta,
$$
where \( \mathcal{A}(x) = \mathbb{I}\{s_{f_D}(x) \leq \tau\} \) and \( M = [\tau - L_s \beta, \tau + L_s \beta] \).

\paragraph{First inequality.}
Let \( s(x) := s_{f_D}(x) \) and \( s'(x) := s_{f_{D^{\setminus i}}}(x) \). Since the only difference in the adversary’s behavior arises from training on or excluding \( z_i \), the output of \( \mathcal{A} \) can only change if the score lies near the threshold. We formalize this as:
\begin{align}
&\left| \Pr[\mathcal{A}(z_i) = 1] - \Pr[\mathcal{A}(z) = 1] \right| \\
&= \left| \mathbb{E}_{x_i}[\mathbb{I}\{s(x_i) \leq \tau\}] - \mathbb{E}_{x}[\mathbb{I}\{s'(x) \leq \tau\}] \right| \\
&\leq \mathbb{E}_{x \sim \mathcal{P}} \left| \mathbb{I}\{s(x) \leq \tau\} - \mathbb{I}\{s'(x) \leq \tau\} \right| \\
&\leq \mathbb{E}_{x \sim \mathcal{P}} \left[ \mathbb{I}\{ |s(x) - \tau| \leq |s(x) - s'(x)| \} \right] \\
&\leq \Pr\left[ |s(x) - \tau| \leq L_s \cdot \| \ell(f_D, \cdot) - \ell(f_{D^{\setminus i}}, \cdot) \|_\infty \right] \\
&\leq \Pr\left[ s_{f_D}(x) \in [\tau - L_s \beta, \tau + L_s \beta] \right] = \Pr[s_{f_D}(x) \in M],
\end{align}
where we used the Lipschitz assumption on the score and the uniform stability bound \( \| \ell(f_D, \cdot) - \ell(f_{D^{\setminus i}}, \cdot) \|_\infty \leq \beta \).

\paragraph{Second inequality.}
The bound \( \Pr[s_{f_D}(x) \in M] \leq 2 Q L_s \beta \) requires a regularity condition on the distribution of the score \( s_{f_D}(x) \). We assume \( s_{f_D}(x) \) admits a probability density function \( p \) bounded above by some constant \( Q \), then
\begin{equation}
\Pr[s_{f_D}(x) \in M] \leq Q \cdot |M| = 2Q L_s \beta.
\end{equation}
Alternatively, the inequality may be interpreted in a worst-case sense, assuming that the measure of any margin region of width \( 2L_s \beta \) is bounded proportionally.

In our settings, the assumption is satisfied in practice. For diffusion models trained via the EDM objective, the score \( s_{f_D}(x) \) is the expected denoising error, which is a smooth, noise-averaged functional of the input and thus likely admits a bounded density on \( \mathbb{R}_+ \). For GANs, the score is typically the (logit) output of the discriminator, which is a continuous function of \( x \) and similarly expected to induce a smooth distribution. In both cases, the bounded-density assumption required for the margin bound holds in practice.


\section{Additional Discussion on Theorem~\ref{thm:adv-bound}}
\label{app:th1tradeoff}

The bound established in Theorem~\ref{thm:adv-bound} states that, under uniform stability and Lipschitz continuity of the score with respect to the loss, the membership advantage of any threshold-based adversary is bounded as
\begin{equation}
\ADV_{\MIA} \leq 2 Q L_s \beta.
\end{equation}
To ensure that this bound is non-trivial (i.e., strictly less than 1), it is necessary that \(2 Q L_s \beta < 1\). This condition introduces a trade-off between the score sensitivity \(L_s\), the stability \(\beta\), and the score density upper bound \(Q\), which we now analyze in more detail.

\vspace{0.5em}
\paragraph{Stability \(\beta\).} The uniform stability parameter \(\beta\) quantifies how much the loss \(\ell(f_D, z)\) changes when one training point is removed from the dataset. For DP-SGD with per-sample gradient clipping at norm \(C\), the expected uniform stability can be upper bounded as
\begin{equation}
\beta = \mathcal{O}\left(\frac{1}{m} \sum_{t=1}^T \alpha_t\right),
\end{equation}
where \(m\) is the dataset size, \(T\) is the number of training steps, and \(\alpha_t\) are the learning rates. Hence, increasing \(m\) or decaying the learning rate can help reduce \(\beta\), thereby tightening the membership advantage bound.

\vspace{0.5em}
\paragraph{Lipschitz constant \(L_s\).} The constant \(L_s\) reflects the sensitivity of the attack score to changes in the loss. Its value depends on the model architecture and the type of score used by the attacker:
\begin{itemize}
    \item In diffusion models, the score is the per-sample denoising error \(s_f(x) = \|\epsilon_\theta(x + \sigma \epsilon, \sigma) - \epsilon\|^2\), directly derived from the loss. However, due to the multiplicative noise and squared error scaling, \(L_s\) can be large, especially when the noise level \(\sigma\) is small.
    \item In GANs, the attack score is the raw discriminator logit \(D(x)\), and the loss is binary cross-entropy with logits. Thus, \(L_s\) corresponds to the inverse derivative of the sigmoid and may remain moderate depending on the activation range of the discriminator.
\end{itemize}
Large \(L_s\) weakens the bound and may dominate the overall expression when the score is highly sensitive to training perturbations.

\vspace{0.5em}
\paragraph{Density bound \(Q\).} The constant \(Q\) assumes that the score \(s_f(x)\) admits a probability density function bounded above by \(Q\). This assumption holds for most smooth neural networks with continuous outputs, and \(Q\) reflects the worst-case concentration of the score distribution. In practice, \(Q\) is often moderate unless the score is extremely peaked.

\vspace{0.5em}
\paragraph{Implication.} The bound \(\ADV_{\MIA} \leq 2 Q L_s \beta\) is informative when all three factors are controlled. In particular, for a fixed model class, reducing \(\beta\) via larger datasets or improved stability (e.g., via regularization or differential privacy) is essential to keep the membership advantage small. At the same time, careful design of the score function (e.g., smooth denoising metrics, logit clipping) may help reduce \(L_s\). This trade-off reflects the fundamental connection between algorithmic stability and the susceptibility of a model to inference attacks.

\section{Additional Comments on Common Randomness (Lemma 2)}
\label{cr}
\textbf{Coupled noise.} Uniform stability compares two runs of DP--SGD on neighbouring datasets
\(D\) and \(D^{\setminus i}\) that differ in one example.
To isolate the \emph{data} effect, we \emph{couple} the randomness:
the two executions use the \textbf{same} mini--batches
\(B_t\) and the \textbf{same} Gaussian noise vectors
\(\eta_t\sim\mathcal N(0,\sigma^2 I)\) for every step~\(t\).
With this coupling, the parameter updates are
\begin{equation}
\theta_{t+1}=\theta_t-\alpha_t\!\bigl(g_t+\eta_t\bigr),
\qquad
\theta'_{t+1}=\theta'_t-\alpha_t\!\bigl(g'_t+\eta_t\bigr),
\end{equation}
where
\(g_t:=\tfrac1b\sum_{j\in B_t}\!\operatorname{clip}\bigl(\nabla\ell(f_{\theta_t}, z_j),C\bigr)\)
and likewise for \(g'_t\).
Let \(\Delta_t:=\theta_t-\theta'_t\).
Because the noise terms cancel,
\begin{equation}
\Delta_{t+1}
=\Delta_t-\alpha_t\,(g_t-g'_t),
\end{equation}
and thus \(\|\Delta_{t+1}\|\le\|\Delta_t\|+\alpha_t\,C/m\),
since the differing example appears in the batch with probability \(b/m\).
Iterating over \(T\) steps yields
\begin{equation}
\|\Delta_T\|\;\le\;\frac{C}{m}\sum_{t=1}^T\alpha_t.
\end{equation}
If \(\ell(\cdot;z)\) is \(L\)-Lipschitz in~\(\theta\),
\begin{equation}
|\ell(f_D,z)-\ell(f_{D^{\setminus i}},z)|
\;\le\;
L\,\|\Delta_T\|
\;\le\;
\frac{LC}{m}\sum_{t=1}^T\alpha_t,
\end{equation}
which gives the classical bound
\(
\displaystyle
\beta\le\frac{2LC}{m}\sum_{t=1}^T\alpha_t.
\)

\medskip\noindent
\textbf{Uncoupled noise.}
If the two runs draw \emph{independent} noise
\(\eta_t\) and \(\eta'_t\),
then
\(
\Delta_{t+1}=\Delta_t-\alpha_t(g_t-g'_t)-\alpha_t(\eta_t-\eta'_t),
\)
so \(\|\Delta_T\|\) acquires an additional
random-walk term of order \(\alpha\sigma\sqrt{T}\).
A typical uncoupled bound therefore becomes
\begin{equation}
\beta \leq
\frac{2LC}{m}\sum_{t=1}^T\alpha_t+\mathcal{O}(\alpha\sigma\sqrt{T}\bigr),
\end{equation}
explicitly reflecting the dependence on the DP noise scale~\(\sigma\).

\section{Implementation Details}
\label{expeap}
For GANs, we follow the training procedure and hyperparameter configuration from \citet{bie2023privategansrevisited}, including the same discriminator and generator architectures, optimizer settings, and training schedule. For diffusion models, we build upon the DP EDM-based setup introduced by \citet{dockhorn2023differentially}, with minor architectural simplifications to ensure stable training on a single GPU. Specifically, we reduce the base number of channels from 128 to 32, use fewer residual blocks per resolution (2 instead of 4), adopt a simplified channel multiplier schedule of $[1, 1, 1, 1]$, and set the embedding channel multiplier to $4$. We also restrict attention to the lowest spatial resolution ($4 \times 4$ instead of $16 \times 16$). We use a fixed dropout rate of 0.1 and a batch size of 128. The model is trained using DP-SGD for 300 epochs with a learning rate of 0.0003 across all privacy levels ($\varepsilon \in \{\infty, 10, 5, 1\}$). To improve training signal, we use a noise multiplicity of 32 loss terms per image, sampled at varying noise levels. We train 20 independent shadow models. At inference, we generate samples using 150 denoising steps with sampling parameters $t_{\min} = 0.002$, $t_{\max} = 80$, $\rho = 7.0$, and guidance scale $3.0$.

\section{Attack Evaluation}

\label{app:attackeval}

To assess the effectiveness of membership inference attacks, we report five standard classification metrics: accuracy, precision, true positive rate (TPR), false positive rate (FPR), and area under the ROC curve (AUC). Accuracy measures the overall proportion of correctly classified examples (both members and non-members). Precision reflects the proportion of true members among the samples predicted as members, capturing the attacker's confidence in positive predictions. TPR (also known as recall or sensitivity) quantifies the fraction of true members correctly identified by the attack. FPR measures the fraction of non-members that are incorrectly predicted as members, and should ideally remain low. Finally, AUC evaluates the attack’s ability to distinguish members from non-members across all possible thresholds; it is a threshold-independent metric where a value of 0.5 corresponds to random guessing. Higher values of accuracy, precision, TPR, and AUC indicate stronger attack performance, whereas lower FPR values are preferable.

\section{Notations}

\begin{table}[H]
\centering
\label{tab:notations}
\begin{tabular}{lll}
\toprule
\textbf{Symbol} & \textbf{Type} & \textbf{Description} \\
\midrule
$\mathcal{X}$ & Space & Input space (e.g., images) \\
$\mathcal{Y} \subset \mathbb{R}$ & Space & Output space (e.g., logits, scores) \\
$\mathcal{P}$ & Dist. & Data distribution \\
$D = \{x_i\}_{i=1}^m$ & Dataset & Training set of size $m$ \\
$D^{\setminus i}$ & Dataset & $D$ without the $i$-th point \\
$f$ & Alg. & Learner $f : \mathcal{X}^m \to \mathcal{F}$ \\
$f_D$ & Model & Model trained on $D$ \\
$\mathcal{F}$ & Space & Hypothesis class (e.g., denoisers) \\
\midrule
$s_f(x)$ & Score & Scalar attack score on $x$ \\
$\ell(f, x)$ & Loss & Per-sample training loss \\
$\mathcal{A}(x)$ & Attack & Binary decision from $s_f(x)$ \\
$\ADV_{\mathrm{MIA}}$ & Metric & Membership advantage \\
$\mathcal{C}(s)$ & Attack & Classifier \\
\midrule
$D_\psi(x)$ & GAN & Discriminator logit \\
$G_\phi(z)$ & GAN & Generator output from noise $z$ \\
$\epsilon_\theta(x_\sigma, \sigma)$ & Diff. & Denoising network \\
$x_\sigma = x_0 + \sigma \epsilon$ & Diff. & Noisy input (forward process) \\
$\lambda(\sigma)$ & Diff. & EDM weighting function \\
$K$ & Diff. & Noise multiplicity (passes per sample) \\
$\bar{\lambda}$ & Diff. & Average EDM weight \\
$B$ & Diff. & Upper bound on prediction error \\
\midrule
$(\varepsilon, \delta)$ & DP & Privacy parameters \\
$\sigma$ & DP/Diff. & Noise scale (context-dependent) \\
$C$ & DP & Gradient clipping norm \\
$\alpha_t$ & DP & Learning rate at step $t$ \\
$T$ & DP & Number of training steps \\
$b$ & DP & Mini-batch size \\
$B_t$ & DP & Batch at iteration $t$ \\
$\eta_t$ & DP & Gaussian noise at step $t$ \\
\midrule
$\beta$ & Stability & Uniform stability coefficient \\
$L_s$ & Stability & Lipschitz const. (score vs loss) \\
$L, L_G, L_D$ & Stability & Lipschitz const. of loss wrt params \\
\bottomrule
\end{tabular}
\caption{Summary of Notations}
\end{table}

\end{document}